\documentclass[letterpaper]{article} 
\usepackage[draft]{aaai2026}  
\usepackage{times}  
\usepackage{helvet}  
\usepackage{courier}  
\usepackage[hyphens]{url}  
\usepackage{graphicx} 
\usepackage{natbib}  
\usepackage{caption} 
\usepackage{algorithm}
\usepackage{algorithmic}

\usepackage{newfloat}
\usepackage{listings}
\DeclareCaptionStyle{ruled}{labelfont=normalfont,labelsep=colon,strut=off} 
\floatstyle{ruled}
\newfloat{listing}{tb}{lst}{}
\floatname{listing}{Listing}
\title{Foundations of Interpretable Models}

\title{Foundations of Interpretable Models}
\author {
    Pietro Barbiero\textsuperscript{\rm 1},
    Mateo Espinosa Zarlenga\textsuperscript{\rm 2},
    Alberto Termine\textsuperscript{\rm 3},\\
    Mateja Jamnik\equalcontriblast\textsuperscript{\rm 2},
    Giuseppe Marra\equalcontriblast\textsuperscript{\rm 4}
}
\affiliations {
    \textsuperscript{\rm 1}IBM Research (CH),
    \textsuperscript{\rm 2}University of Cambridge (UK),
    \textsuperscript{\rm 3}SUPSI, IDSIA (CH),
    \textsuperscript{\rm 4}KU Leuven (BE)\\
    pietro.barbiero@ibm.com, 
    \{me466,mj201\}@cam.ac.uk,
    alberto.termine@supsi.ch,
    giuseppe.marra@kuleuven.be
}

\usepackage{footmisc}

\usepackage{adjustbox}
\usepackage{graphicx}
\usepackage{tikz-cd}

\usepackage{amssymb}
\usepackage{amsmath}

\usepackage{amsthm}
\theoremstyle{plain}

\usepackage{mdframed}

\newtheorem*{lemma*}{Lemma}
\newtheorem*{proposition*}{Proposition}
\newtheorem*{corollary*}{Corollary}
\newtheorem*{conjecture*}{Conjecture}
\newmdtheoremenv[
  backgroundcolor=gray!20,   
  linecolor=gray!50,         
  innertopmargin=2pt,
  innerbottommargin=2pt,
  skipabove=2pt,
  skipbelow=2pt
]{theorem}{Theorem}

\theoremstyle{definition}
\newtheorem{example}{Example}

\newtheorem*{definition*}{Definition}
\newtheorem*{example*}{Example}
\newtheorem*{prob*}{Problem}
\newtheorem*{remark*}{Remark}
\newtheorem*{notation*}{Notation}
\newtheorem*{exer*}{Exercise}

\newmdtheoremenv[
  backgroundcolor=gray!20,   
  linecolor=gray!50,         
  innertopmargin=2pt,
  innerbottommargin=2pt,
  skipabove=2pt,
  skipbelow=2pt
]{definition}{Definition}

\usepackage{mathtools}

\usepackage{tikz}
\usetikzlibrary{arrows.meta, positioning, quotes, fit, shapes, backgrounds, calc}

\newlength{\nodesep}
\newlength{\innersep}
\newlength{\sep}
\tikzset{
    observed/.style={circle, draw, thick, fill=gray!50, minimum size=.7cm},
    unobserved/.style={circle, draw, thick, minimum size=.7cm},
    blank/.style={circle, draw, thick, dotted, minimum size=.7cm, opacity=0.5},
    square1/.style={draw, fill=gray!20, minimum size=.3cm},
    square2/.style={draw, fill=gray!50, minimum size=.3cm},
    square3/.style={draw, fill=gray!100, minimum size=.3cm},
    true/.style = {circle, draw=none, thick, fill=green!50, minimum width=.5cm, minimum height=.5cm, inner sep=0pt},
    false/.style = {circle, draw=none, thick, fill=red!50, minimum width=.5cm, minimum height=.5cm, inner sep=0pt},
    plain/.style = {circle, draw, thick, fill=none, minimum width=.5cm, minimum height=.5cm, inner sep=0pt},
    arrowstyle/.style={->, thick, rounded corners},
    arrowstyledash/.style={->, thick, rounded corners, dash pattern=on 2pt off .6pt, gray},
    adjweight/.style={->, dashed, line width=.1mm, colour=blue!50},
    adjweightbig/.style={->, line width=.6mm, colour=blue!50},
    cembtrue1/.style={draw, fill=green!20, minimum size=.3cm},
    cembtrue2/.style={draw, fill=green!50, minimum size=.3cm},
    cembtrue3/.style={draw, fill=green!100, minimum size=.3cm},
    cembfalse1/.style={draw, fill=red!20, minimum size=.3cm},
    cembfalse2/.style={draw, fill=red!50, minimum size=.3cm},
    cembfalse3/.style={draw, fill=red!100, minimum size=.3cm},
    operation/.style={draw, thick, fill=blue!20, minimum size=0.6cm,  rounded corners=5pt, inner sep=1pt, outer sep=0pt, align=center},
}

\newcommand{\exampleConceptRound}{
\begin{tikzpicture}[
    node distance = .3cm,
]

\node[circle, draw, thick, fill=gray!30, fill opacity=0.3, minimum size=3cm] (models) {};

\node[circle, draw, thick, fill=gray!30, fill opacity=0.3, minimum size=3cm, right=1cm of models] (sentences) {};

\node (apple) at (-.3,.3) {\inlineimg{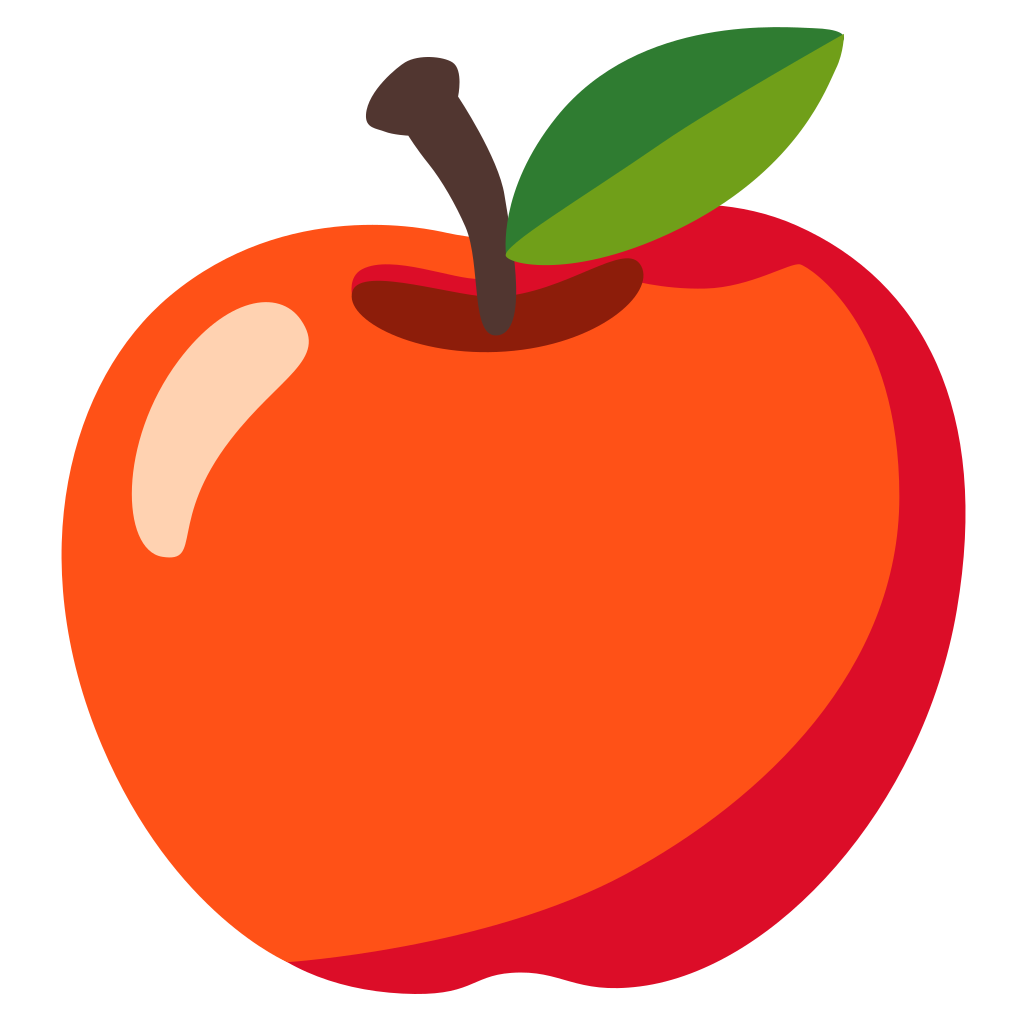}};
\node (donut) at (.3,.5) {\inlineimg{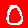}};
\node (ball) at (.4,-.1) {\inlineimg{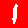}};
\node (avocado) at (-.3,-.9) {\inlineimg{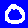}};
\node (avocado) at (.4,-1.) {\inlineimg{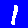}};

\node (green) at (4.7,-.1) {\tiny $\texttt{one} \wedge \neg \texttt{fruit}$};
\node (round) at (3.9,.8) {\texttt{red}};
\node (sweet) at (3.3,-.2) {\texttt{zero}};
\node (soft) at (4.4,-.7) {\texttt{even}};

\node (objects) at (-1.5, 1) {$U$};
\node (sentences) at (5.5,1) {$S$};

\begin{pgfonlayer}{background}
    \node[draw, thick, fill=gray!30, rounded corners, inner sep=.2cm, yshift=.1cm, fit=(apple) (donut) (ball), label=above:{$M$}] (M) {};

    \node[draw, thick, fill=gray!30, rounded corners, inner sep=.cm, yshift=.cm, fit=(round), label=right:{$T$}] (T) {};

    \draw[arrowstyledash, red] (T)  to [bend left=20] node[midway, below] {$\beta$} (M);
    \draw[arrowstyledash, red] (M)  to [bend left=20] node[midway, above] {$\gamma$} (T);
\end{pgfonlayer}

\end{tikzpicture}
}

\tikzstyle{none}=[]
\tikzstyle{white dot}=[fill=white, draw=black, shape=circle, text width=.1cm, inner sep=.03cm]
\tikzstyle{red dot}=[fill={rgb,255: red,163; green,0; blue,2}, draw=black, shape=circle]
\tikzstyle{medium box}=[fill=white, draw=black, shape=rectangle, minimum width=2cm, minimum height=1.cm]
\tikzstyle{small box}=[fill=white, draw=black, shape=rectangle, minimum width=.5cm, minimum height=.5cm]
\tikzstyle{triangle}=[fill=white, draw=black, regular polygon, regular polygon sides=3, shape border rotate=90, minimum width=1.cm, text width=.3cm, inner sep=.0cm, align=center, anchor=center]

\tikzstyle{new edge style 0}=[-, fill=none, draw={rgb,255: red,0; green,172; blue,187}]
\tikzstyle{thick edge}=[-, line width=.03cm]

\pgfdeclarelayer{nodelayer}
\pgfdeclarelayer{edgelayer}
\pgfsetlayers{nodelayer,background,main,edgelayer}

\newcommand{\compositionalElements}{
\begin{tikzpicture}
	\begin{pgfonlayer}{nodelayer}
		\node [style=medium box] (5) at (-4.75, 7) {$f$};
		\node [style=none] (7) at (-5.75, 7.25) {};
		\node [style=none] (9) at (-5.75, 6.75) {};
		\node [style=none] (11) at (-3.75, 7) {};
		\node [style=none] (12) at (-6.5, 7.25) {};
		\node [style=none] (14) at (-6.5, 6.75) {};
		\node [style=none] (15) at (-3, 7) {};
		\node [style=white dot] (16) at (-1.25, 7) {};
		\node [style=none] (17) at (-2, 7) {};
		\node [style=none] (18) at (-0.75, 7.5) {};
		\node [style=none] (19) at (-0.75, 6.5) {};
		\node [style=none] (20) at (1.75, 7) {};
		\node [style=none] (21) at (0.75, 7) {};
		\node [style=none] (22) at (1.75, 6.625) {};
		\node [style=none] (23) at (1.75, 7.375) {};
		\node [style=none] (24) at (1.875, 6.75) {};
		\node [style=none] (25) at (1.875, 7.25) {};
		\node [style=none] (26) at (2, 6.875) {};
		\node [style=none] (27) at (2, 7.125) {};
		\node [style=none] (32) at (-6.4, 7) {$\dots$};
		\node [style=triangle] (33) at (3.75, 6.975) {$k$};
		\node [style=none] (34) at (4, 6.975) {};
		\node [style=none] (35) at (4.75, 6.975) {};
		\node [style=none] (39) at (-4.75, 5.75) {Single output box};
		\node [style=none] (40) at (-1.25, 5.75) {Copy};
		\node [style=none] (41) at (1.5, 5.75) {Discard};
		\node [style=none] (42) at (4, 5.75) {Constant};
	\end{pgfonlayer}
	\begin{pgfonlayer}{edgelayer}
		\draw (12.center) to (7.center);
		\draw (14.center) to (9.center);
		\draw [in=180, out=0, looseness=1.25] (11.center) to (15.center);
		\draw (17.center) to (16);
		\draw (21.center) to (20.center);
		\draw [style=thick edge] (23.center) to (22.center);
		\draw [style=thick edge] (25.center) to (24.center);
		\draw [style=thick edge] (27.center) to (26.center);
		\draw [in=180, out=0, looseness=1.25] (34.center) to (35.center);
		\draw [in=90, out=180] (18.center) to (16);
		\draw [in=-90, out=180] (19.center) to (16);
	\end{pgfonlayer}
\end{tikzpicture}

}

\newcommand{\compositionalExample}{
\begin{tikzpicture}
	\begin{pgfonlayer}{nodelayer}
		\node [style=medium box] (42) at (9.75, 16.5) {$C_1 \vee C_4$};
		\node [style=none] (45) at (10.75, 16.5) {};
		\node [style=none, label={right:$Y$}] (46) at (11, 16.5) {};
		\node [style=medium box] (47) at (7.25, 15.375) {$\neg C_2 \wedge C_{3}$};
		\node [style=none] (49) at (8.75, 16.25) {};
		\node [style=none] (52) at (8.25, 15.375) {};
		\node [style=none, label={right:$C_4$}] (53) at (8.25, 15.625) {};
		\node [style=medium box] (54) at (4.75, 17.25) {$P(C_1 \mid X)$};
		\node [style=none] (55) at (3.75, 17.25) {};
		\node [style=medium box] (57) at (4.75, 16) {$P(C_2 \mid X)$};
		\node [style=none] (58) at (3.75, 16) {};
		\node [style=none] (59) at (5.75, 16) {};
		\node [style=small box] (60) at (2.75, 16) {$X$};
		\node [style=white dot] (61) at (3.25, 16) {};
		\node [style=none] (63) at (8.75, 16.75) {};
		\node [style=medium box] (66) at (4.75, 14.75) {$P(C_3 \mid X)$};
		\node [style=none] (67) at (3.75, 14.75) {};
		\node [style=none] (68) at (5.75, 14.75) {};
		\node [style=none] (69) at (6.25, 15.625) {};
		\node [style=none] (71) at (6.25, 15.125) {};
		\node [style=none] (74) at (-3.5, 17.25) {};
		\node [style=medium box] (75) at (-1.625, 16) {$C_1 \vee (\neg C_2 \wedge C_3)$};
		\node [style=none] (76) at (-0.225, 16) {};
		\node [style=none, label={right:$Y$}] (77) at (0.025, 16) {};
		\node [style=medium box] (78) at (-4.5, 17.25) {$P(C_1 \mid X)$};
		\node [style=none] (79) at (-5.5, 17.25) {};
		\node [style=none] (80) at (-3.5, 17.25) {};
		\node [style=medium box] (81) at (-4.5, 16) {$P(C_2 \mid X)$};
		\node [style=none] (82) at (-5.5, 16) {};
		\node [style=none] (83) at (-3.5, 16) {};
		\node [style=small box] (84) at (-6.5, 16) {$X$};
		\node [style=white dot] (85) at (-6, 16) {};
		\node [style=medium box] (86) at (-4.5, 14.75) {$P(C_3 \mid X)$};
		\node [style=none] (87) at (-5.5, 14.75) {};
		\node [style=none] (88) at (-3.5, 14.75) {};
		\node [style=none] (89) at (-3, 16.25) {};
		\node [style=none] (90) at (-3, 16) {};
		\node [style=none] (91) at (-3, 15.75) {};
		\node [style=none, label={right:{\Huge$\equiv$}}] (92) at (1, 16) {};
		\node [style=none] (93) at (5.75, 17.25) {};
	\end{pgfonlayer}
	\begin{pgfonlayer}{edgelayer}
		\draw (45.center) to (46.center);
		\draw [in=180, out=0, looseness=0.75] (52.center) to (49.center);
		\draw (60) to (61);
		\draw [bend right=15, looseness=0.75] (55.center) to (61);
		\draw (58.center) to (61);
		\draw [bend right=15, looseness=0.75] (61) to (67.center);
		\draw (76.center) to (77.center);
		\draw (84) to (85);
		\draw [bend right=15, looseness=0.75] (79.center) to (85);
		\draw (82.center) to (85);
		\draw [bend right=15, looseness=0.75] (85) to (87.center);
		\draw [in=180, out=0, looseness=0.50] (80.center) to (89.center);
		\draw (83.center) to (90.center);
		\draw [in=180, out=0, looseness=0.50] (88.center) to (91.center);
		\draw [in=180, out=0, looseness=0.75] (93.center) to (63.center);
		\draw [in=180, out=0, looseness=0.75] (59.center) to (69.center);
		\draw [in=180, out=0, looseness=0.75] (68.center) to (71.center);
	\end{pgfonlayer}
\end{tikzpicture}
}

\newcommand{\expressivityParametricJoint}{
\begin{tikzpicture}
	\begin{pgfonlayer}{nodelayer}
		\node [style=medium box] (5) at (-4.25, 7) {$P(C \mid X)$};
		\node [style=none] (9) at (-5.25, 7) {};
		\node [style=none] (11) at (-3.25, 7) {};
		\node [style=none] (15) at (-2.85, 6.5) {};
		\node [style=none] (42) at (-5.5, 6) {};
		\node [style=none] (43) at (-2.825, 6) {};
		\node [style=small box] (44) at (-5.75, 7) {$X$};
		\node [style=medium box] (47) at (-1.825, 6.25) {$P(Y \mid C; \Theta)$};
		\node [style=none] (50) at (-0.825, 6.25) {};
		\node [style=none, label={right:$Y$}] (51) at (-0.575, 6.25) {};
		\node [style=small box] (54) at (-5.75, 6) {$\Theta$};
	\end{pgfonlayer}
	\begin{pgfonlayer}{edgelayer}
		\draw [in=180, out=0, looseness=1.25] (11.center) to (15.center);
		\draw [in=180, out=0, looseness=1.50] (42.center) to (43.center);
		\draw [in=180, out=0, looseness=1.25] (50.center) to (51.center);
		\draw (44) to (9.center);
	\end{pgfonlayer}
\end{tikzpicture}

}

\newcommand{\expressivityDCR}{
\begin{tikzpicture}
	\begin{pgfonlayer}{nodelayer}
		\node [style=medium box] (5) at (-4.55, 6.5) {$P(C \mid X)$};
		\node [style=none] (9) at (-5.55, 6.5) {};
		\node [style=none] (11) at (-3.55, 6.5) {};
		\node [style=none] (15) at (-3.05, 6) {};
		\node [style=medium box] (40) at (-4.55, 5) {$P(\Theta \mid X)$};
		\node [style=none] (41) at (-5.55, 5) {};
		\node [style=none] (42) at (-3.55, 5) {};
		\node [style=none] (43) at (-3.05, 5.5) {};
		\node [style=small box] (44) at (-6.5, 5.75) {$X$};
		\node [style=white dot] (45) at (-5.85, 5.75) {};
		\node [style=medium box] (47) at (-2.025, 5.75) {$P(Y \mid C; \Theta)$};
		\node [style=none] (50) at (-1, 5.75) {};
		\node [style=none, label={right:$Y$}] (51) at (-0.75, 5.75) {};
	\end{pgfonlayer}
	\begin{pgfonlayer}{edgelayer}
		\draw [in=180, out=0, looseness=1.25] (11.center) to (15.center);
		\draw [in=180, out=0, looseness=1.25] (42.center) to (43.center);
		\draw (44) to (45);
		\draw [in=180, out=0, looseness=1.25] (50.center) to (51.center);
		\draw [in=90, out=-150] (9.center) to (45);
		\draw [in=-90, out=150] (41.center) to (45);
	\end{pgfonlayer}
\end{tikzpicture}

}

\newcommand{\expressivityCMR}{
\begin{tikzpicture}
	\begin{pgfonlayer}{nodelayer}
		\node [style=medium box] (5) at (-5.95, 7) {$P(C \mid X)$};
		\node [style=none] (11) at (-4.95, 7) {};
		\node [style=none] (15) at (-1.4, 6.25) {};
		\node [style=medium box] (40) at (-5.95, 5.5) {$P(R \mid X)$};
		\node [style=none] (42) at (-2.15, 5) {};
		\node [style=none] (43) at (-1.4, 5.75) {};
		\node [style=medium box] (47) at (-0.375, 6) {$P(Y \mid C; \Theta)$};
		\node [style=none] (50) at (0.65, 6) {};
		\node [style=none, label={right:$Y$}] (51) at (0.9, 6) {};
		\node [style=small box] (52) at (-7.75, 4.75) {$Q_\Theta$};
		\node [style=medium box] (53) at (-3.3, 5) {$P(\Theta \mid R, Q_\Theta)$};
		\node [style=none] (54) at (-7.35, 4.75) {};
		\node [style=none] (55) at (-4.45, 4.75) {};
		\node [style=none] (56) at (-4.95, 5.5) {};
		\node [style=none] (57) at (-4.45, 5.25) {};
		\node [style=none] (58) at (-6.95, 7) {};
		\node [style=none] (59) at (-6.95, 5.5) {};
		\node [style=small box] (60) at (-7.9, 6.25) {$X$};
		\node [style=white dot] (61) at (-7.25, 6.25) {};
	\end{pgfonlayer}
	\begin{pgfonlayer}{edgelayer}
		\draw [in=180, out=0, looseness=1.25] (11.center) to (15.center);
		\draw [in=180, out=0, looseness=1.25] (42.center) to (43.center);
		\draw [in=180, out=0, looseness=1.25] (50.center) to (51.center);
		\draw [in=180, out=0, looseness=1.25] (54.center) to (55.center);
		\draw [in=180, out=0, looseness=1.25] (56.center) to (57.center);
		\draw (60) to (61);
		\draw [in=90, out=-150] (58.center) to (61);
		\draw [in=-90, out=150] (59.center) to (61);
	\end{pgfonlayer}
\end{tikzpicture}

}

\newcommand{\dointerventionExample}{
\begin{tikzpicture}
	\begin{pgfonlayer}{nodelayer}
		\node [style=none] (0) at (3.9, 4.75) {};
		\node [style=medium box] (1) at (6.4, 4.5) {$P(Y \mid C)$};
		\node [style=none] (2) at (7.4, 4.5) {};
		\node [style=none, label={right:$Y$}] (3) at (7.65, 4.5) {};
		\node [style=none] (4) at (5.4, 4.25) {};
		\node [style=none] (5) at (3.9, 4.75) {};
		\node [style=none] (6) at (3.9, 4.375) {};
		\node [style=none] (7) at (3.9, 5.125) {};
		\node [style=none] (8) at (4.025, 4.5) {};
		\node [style=none] (9) at (4.025, 5) {};
		\node [style=none] (10) at (4.15, 4.625) {};
		\node [style=none] (11) at (4.15, 4.875) {};
		\node [style=triangle] (12) at (4.9, 4.75) {$k$};
		\node [style=none] (13) at (5.4, 4.75) {};
		\node [style=medium box] (14) at (2.4, 5.5) {$P(C_1 \mid X)$};
		\node [style=none] (16) at (3.4, 5.5) {};
		\node [style=medium box] (17) at (2.4, 3.5) {$P(C_k \mid X)$};
		\node [style=none] (19) at (3.4, 3.5) {};
		\node [style=none] (22) at (2.4, 4.5) {$\dots$};
		\node [style=none] (23) at (4.15, 6.5) {Do-intervention};
		\node [style=none] (24) at (1.4, 5.5) {};
		\node [style=none] (25) at (1.4, 3.5) {};
		\node [style=small box] (26) at (0.45, 4.5) {$X$};
		\node [style=white dot] (27) at (1.1, 4.5) {};
	\end{pgfonlayer}
	\begin{pgfonlayer}{edgelayer}
		\draw (2.center) to (3.center);
		\draw [style=thick edge] (7.center) to (6.center);
		\draw [style=thick edge] (9.center) to (8.center);
		\draw [style=thick edge] (11.center) to (10.center);
		\draw (12) to (13.center);
		\draw [in=180, out=0, looseness=0.75] (16.center) to (5.center);
		\draw [in=180, out=0] (19.center) to (4.center);
		\draw (26) to (27);
		\draw [in=90, out=-150] (24.center) to (27);
		\draw [in=-90, out=150] (25.center) to (27);
	\end{pgfonlayer}
\end{tikzpicture}

}

\newcommand{\gtinterventionExample}{

\begin{tikzpicture}
	\begin{pgfonlayer}{nodelayer}
		\node [style=none] (19) at (-2.25, 4.75) {};
		\node [style=medium box] (42) at (0.25, 4.5) {$P(Y \mid C)$};
		\node [style=none] (45) at (1.25, 4.5) {};
		\node [style=none, label={right:$Y$}] (46) at (1.5, 4.5) {};
		\node [style=none] (49) at (-0.75, 4.25) {};
		\node [style=none] (54) at (-2.25, 4.75) {};
		\node [style=none] (55) at (-2.25, 4.375) {};
		\node [style=none] (56) at (-2.25, 5.125) {};
		\node [style=none] (57) at (-2.125, 4.5) {};
		\node [style=none] (58) at (-2.125, 5) {};
		\node [style=none] (59) at (-2, 4.625) {};
		\node [style=none] (60) at (-2, 4.875) {};
		\node [style=small box] (61) at (-1.25, 4.75) {$C_1$};
		\node [style=none] (62) at (-0.75, 4.75) {};
		\node [style=medium box] (63) at (-3.75, 5.5) {$P(C_1 \mid X)$};
		\node [style=none] (65) at (-2.75, 5.5) {};
		\node [style=medium box] (66) at (-3.75, 3.5) {$P(C_k \mid X)$};
		\node [style=none] (68) at (-2.75, 3.5) {};
		\node [style=none] (71) at (-3.75, 4.5) {$\dots$};
		\node [style=none] (72) at (-2, 6.5) {Ground-truth intervention};
		\node [style=none] (73) at (-4.75, 5.5) {};
		\node [style=none] (74) at (-4.75, 3.5) {};
		\node [style=small box] (75) at (-5.7, 4.5) {$X$};
		\node [style=white dot] (76) at (-5.05, 4.5) {};
	\end{pgfonlayer}
	\begin{pgfonlayer}{edgelayer}
		\draw (45.center) to (46.center);
		\draw [style=thick edge] (56.center) to (55.center);
		\draw [style=thick edge] (58.center) to (57.center);
		\draw [style=thick edge] (60.center) to (59.center);
		\draw (61) to (62.center);
		\draw [in=180, out=0, looseness=0.75] (65.center) to (54.center);
		\draw [in=180, out=0] (68.center) to (49.center);
		\draw (75) to (76);
		\draw [in=90, out=-150] (73.center) to (76);
		\draw [in=-90, out=150] (74.center) to (76);
	\end{pgfonlayer}
\end{tikzpicture}
}

\newcommand{\doInterventionSmall}{
\begin{tikzpicture}
	\begin{pgfonlayer}{nodelayer}
		\node [style=none] (0) at (3.9, 4.75) {};
		\node [style=none] (5) at (3.9, 4.75) {};
		\node [style=none] (6) at (3.9, 4.375) {};
		\node [style=none] (7) at (3.9, 5.125) {};
		\node [style=none] (8) at (4.025, 4.5) {};
		\node [style=none] (9) at (4.025, 5) {};
		\node [style=none] (10) at (4.15, 4.625) {};
		\node [style=none] (11) at (4.15, 4.875) {};
		\node [style=triangle] (12) at (4.9, 4.75) {$k$};
		\node [style=none] (13) at (5.4, 4.75) {};
		\node [style=none] (16) at (3.4, 4.75) {};
	\end{pgfonlayer}
	\begin{pgfonlayer}{edgelayer}
		\draw [style=thick edge] (7.center) to (6.center);
		\draw [style=thick edge] (9.center) to (8.center);
		\draw [style=thick edge] (11.center) to (10.center);
		\draw (12) to (13.center);
		\draw [in=180, out=0, looseness=0.75] (16.center) to (5.center);
	\end{pgfonlayer}
\end{tikzpicture}

}

\newcommand{\gtInterventionSmall}{
\begin{tikzpicture}
	\begin{pgfonlayer}{nodelayer}
		\node [style=none] (0) at (3.9, 4.75) {};
		\node [style=none] (5) at (3.9, 4.75) {};
		\node [style=none] (6) at (3.9, 4.375) {};
		\node [style=none] (7) at (3.9, 5.125) {};
		\node [style=none] (8) at (4.025, 4.5) {};
		\node [style=none] (9) at (4.025, 5) {};
		\node [style=none] (10) at (4.15, 4.625) {};
		\node [style=none] (11) at (4.15, 4.875) {};
		\node [style=small box] (12) at (4.65, 4.75) {$C_i$};
		\node [style=none] (13) at (5.15, 4.75) {};
		\node [style=none] (16) at (3.4, 4.75) {};
	\end{pgfonlayer}
	\begin{pgfonlayer}{edgelayer}
		\draw [style=thick edge] (7.center) to (6.center);
		\draw [style=thick edge] (9.center) to (8.center);
		\draw [style=thick edge] (11.center) to (10.center);
		\draw (12) to (13.center);
		\draw [in=180, out=0, looseness=0.75] (16.center) to (5.center);
	\end{pgfonlayer}
\end{tikzpicture}

}

\newcommand{\blueprint}{
\begin{tikzpicture}
	\begin{pgfonlayer}{nodelayer}
		\node [style=medium box] (5) at (-5.75, 6) {$P(C,  \Theta \mid X)$};
		\node [style=none] (9) at (-6.775, 6) {};
		\node [style=small box] (44) at (-7.4, 6) {$X$};
		\node [style=none] (60) at (-0.775, 6.5) {};
		\node [style=none] (61) at (-0.3, 6.5) {};
		\node [style=medium box] (67) at (0.9, 6.5) {$P(Y \mid C_\tau; \Theta_\tau)$};
		\node [style=none] (68) at (-4.7, 6) {};
		\node [style=none] (79) at (2.1, 6.5) {};
		\node [style=none] (80) at (2.35, 6.5) {};
		\node [style=none] (81) at (2.85, 6.5) {$Y$};
		\node [style=small box] (87) at (-7.4, 7) {$P(\tau)$};
		\node [style=none] (89) at (-3.925, 6.75) {};
		\node [style=medium box] (90) at (-2.35, 6.5) {$P(C_\tau, \Theta_\tau \mid C, \Theta,\tau)$};
		\node [style=none] (93) at (-3.925, 6.25) {};
	\end{pgfonlayer}
	\begin{pgfonlayer}{edgelayer}
		\draw [in=180, out=0, looseness=1.25] (60.center) to (61.center);
		\draw [in=180, out=0, looseness=1.25] (79.center) to (80.center);
		\draw [in=180, out=0, looseness=1.50] (87) to (89.center);
		\draw (44) to (9.center);
		\draw [in=-180, out=0] (68.center) to (93.center);
	\end{pgfonlayer}
\end{tikzpicture}

}

\newcommand{\inlineimg}[1]{%
  \smash{\raisebox{-.2\height}{\includegraphics[height=.9em]{figs/#1.png}}}%
}
\newcommand{\inlineimgsmall}[1]{%
  \smash{\raisebox{-.2\height}{\includegraphics[height=.7em]{figs/#1.png}}}%
}
\newcommand{\inlineimgrotate}[1]{%
  \smash{\raisebox{-.2\height}{\includegraphics[height=.9em, angle=90]{figs/#1.png}}}%
}

\lstdefinestyle{python}{
  language=Python,
  basicstyle=\ttfamily\small,            
  keywordstyle=\color{blue}\bfseries,    
  commentstyle=\color{green!60!black}\itshape, 
  stringstyle=\color{orange},            
  identifierstyle=\color{black},         
  showstringspaces=false,                
  numbers=left,                          
  numberstyle=\tiny\color{gray},         
  stepnumber=1,                          
  numbersep=5pt,                         
  rulecolor=\color{gray},                
  breaklines=false,                       
  captionpos=b,                          
  keepspaces=true                        
}

\usepackage{adjustbox}

\usepackage{enumitem}

\begin{document}
\nocopyright

\maketitle

\begin{abstract}
We argue that existing definitions of interpretability are not \textit{actionable} in that they fail to inform users about general, sound, and robust interpretable model design. This makes current interpretability research fundamentally ill-posed. To address this issue, we propose a definition of interpretability that is general, simple, and subsumes existing informal notions within the interpretable AI community. We show that our definition is actionable, as it directly reveals the foundational properties, underlying assumptions, principles, data structures, and architectural features necessary for designing interpretable models. Building on this, we propose a general blueprint for designing interpretable models and introduce the first open-sourced library with native support for interpretable data structures and processes.
\end{abstract}

\begin{links}
    \link{Code}{https://github.com/pyc-team/pytorch_concepts}
\end{links}

\section{Introduction}
Recent years have seen a surge in interpretable models whose decisions can be easily understood by humans. These models now offer a performance comparable to that of powerful black-box models like Deep Neural Networks (DNNs)~\citep{senn,protopnet,cem}, and are increasingly employed to diagnose errors, ensure fairness, and comply with legal standards~\citep{lee2021development, meng2022interpretability}.

In this paper, we argue that current research in interpretable Artificial Intelligence (AI) is ill-posed for two reasons. First, the community has failed to formalise an agreed-upon definition of interpretability. Second, although previous attempts to define interpretability offer some intuition on what one may consider to be ``interpretable AI'',
they remain \textit{unactionable}: it is unclear how they can be directly translated into general design principles for interpretable models. 

For instance, ~\citet{kim2016examples}, \citet{biran2017explanation}, and 
\citet{miller2019explanation} informally suggested that \emph{a method is interpretable if a user can correctly and efficiently predict the method's results}.
More recently,~\citet{murphy2023probabilistic} claimed that \emph{there is no universal, mathematical definition of interpretability, and there never will be}. While mathematical definitions of interpretability exist and have been influential in other fields -- such as in logic systems~\citep{tarski1953undecidable} -- we argue that such rigorous frameworks (1)~often rely on substantial assumptions, and (2)~have not been used to directly deduce consequences for and drive research in interpretable AI.
This lack of a clear, actionable, and contextualised definition imposes a barrier to identifying the key challenges, principles, and architectural features necessary for designing interpretable AI models.



\paragraph{Contributions} This paper formulates AI interpretability as a well-posed problem. We achieve this goal as follows:
\begin{itemize}[topsep=3pt, leftmargin=9pt, itemsep=0pt]
    \item \textbf{We propose a general, simple, and \textit{actionable} definition of interpretability}. We formalise interpretability as \textit{inference equivariance}, defining a function as interpretable if the inference mechanisms of both the function and its user reach the same results given the same inputs. We show that although this definition encompasses existing informal notions of interpretability, directly verifying inference equivariance is intractable (§\ref{sec:equiv}).
    
    \item \textbf{We identify assumptions and principles that make interpretability tractable and draw consequences on model design}. Specifically, we demonstrate the actionability of our definition by pinpointing concrete assumptions, principles, and data structures that make interpretability tractable in practice (§\ref{sec:breaking}, §\ref{sec:concept}, and §\ref{sec:concept-based-eq}). Based on these results, we draw general consequences for the design of interpretable models
    (§\ref{sec:principles}).
    
    \item \textbf{We propose a blueprint for interpretable models}. Building on our definition, we (1)~propose a general modelling paradigm for building interpretable models (§\ref{sec:blueprint}), and (2)~introduce an open-source library with native support for interpretable data structures and processes.
    
\end{itemize}

\section{Interpretability as Inference Equivariance} \label{sec:equiv}
We aim to identify the key challenges, assumptions, and principles underlying interpretability and utilise them to design interpretable models. Therefore, our first objective is to propose an \textit{actionable} definition of interpretability that informs model design.
As a running example, we consider a probabilistic model $P(Y \mid X; m)$ parametrised by an unknown function $m$ that predicts whether an object $\omega \in \Omega$ described by features $X \subseteq \mathbb{R}^D$ belongs to a class $Y\subseteq \mathbb{N}$ (without loss of generality, we assume we work with classification tasks). At this stage, we assume that we observe both $X$ and $Y$, but we do not know yet what $X$ and $Y$ represent. Given a set of example observations (e.g., $\omega \in \{\inlineimg{zerored}, \inlineimg{zeroblue}, \inlineimg{oneblue}\}$), we can describe $m$ via the following table:
\begin{center}
\text{\textbf{Table 1}: Tabular representation of a function $m$.}
\label{eq:cpt}
\end{center}
\begin{align*} 
{\begin{array}{c|cc|c}
\omega & X_1 & X_2 & Y = m(X) \\
\hline
\inlineimg{zerored} & 0 & 1 & 1 \\
\inlineimg{zeroblue} & 0 & 0 & 1 \\
\inlineimg{oneblue} & 1 & 0 & 0 \\
\end{array}}
\end{align*}

In some cases, we might be able to associate a description with some variables. For example, we could associate the strings ``\texttt{one}'' to $X_1$, ``\texttt{red}'' to $X_2$, and ``\texttt{even}'' to $Y$. This association establishes a particular relation between the function $m$ and human knowledge (e.g., number theory). For instance, given the object $\omega=\inlineimg{zeroblue}$ we could either:
\begin{itemize}[topsep=3pt, leftmargin=9pt, itemsep=0pt]
    \item apply $m$ on the object's features $x = X_{1,2}(\omega) = (0, 0)$ to compute $Y=1$, and then translate the result $Y=1$ into ``human terms'', getting $\texttt{even} = \text{yes}$;
    \item or we could translate the object's features $X_{1,2}(\omega) = (0, 0)$ in ``human terms'', getting $(\texttt{one}, \texttt{red}) = (\text{no}, \text{no})$, and then predict parity ourselves to get $\texttt{even} = \text{yes}$.
\end{itemize}
This equivalence between the function $m$ and our inference mechanism can be represented as a commutative diagram:
{\small 
\[
\begin{adjustbox}{max width=\linewidth}
\begin{tikzcd}[column sep=3.5cm, row sep=normal]
(X_1, X_2) = (0, 0) \arrow[r, "\text{unknown function }m"] \arrow[d, "\text{``translate''}"'] & Y=1 \arrow[d, "\text{``translate''}"] \\
{\scriptscriptstyle
\shortstack{
$(\texttt{one}, \texttt{red}) =$
$(\text{no}, \text{no})$
}
} \arrow[r, "\text{human inference } h"']  & \texttt{even}=\text{yes}
\end{tikzcd}
\end{adjustbox}
\]
}
If this diagram commutes for \textit{any} input $x = X(\omega)$ (i.e., if we reach the same result following different paths), then the function $m$ has a one-to-one correspondence with our knowledge. This leads us to an actionable procedure, akin to the so-called Turing test~\cite{turing1950computing}, for establishing whether an unknown function is interpretable. Specifically:
\begin{quote}
\emph{A function is interpretable to a user if the function's and the user's inference mechanisms are equivariant.}
\end{quote}
We formalise this criterion as follows. 
\begin{definition}(Interpretability as inference equivariance) \label{def:infeq}
    A function $m$ is \textbf{interpretable} for a user represented by a function $h$ via a translation $\tau$ \textit{iff} the following diagram commutes for any realisation of $X^{(m)}$:
    \[
    {\begin{adjustbox}{max width=\linewidth}
    \begin{tikzcd}[column sep=normal, row sep=small]
    X^{(m)} \arrow[r, "m"] \arrow[d, "\tau"'] & Y^{(m)} \arrow[d, "\tau"] \\
    X^{(h)} \arrow[r, "h"']  & Y^{(h)}
    \end{tikzcd}
    \end{adjustbox}}
    \]
\end{definition}

\paragraph{Consequences for Interpretability} 
The definition above is effective because: (1)~it subsumes and formalises current informal definitions and intuitions within the interpretable AI community~\citep{kim2016examples,hewittposition} (see §\ref{app:other-def}); (2)~it is significantly simpler than prior formal definitions of interpretability proposed in formal systems~\citep{tarski1953undecidable} and causality~\citep{rubenstein2017causal,geiger2024causal,marconato2023interpretability} as it situates the definition in a typical machine learning context making fewer structural assumptions; and, most importantly, (3)~it is actionable as it enables us to identify concrete consequences that uniquely characterise interpretability in AI. Some of these consequences include (see §\ref{app:inference-eq} for an extended list):
\begin{enumerate}
    \item 
    \textbf{In principle, any function is interpretable.} We ``only'' need a translation $\tau$ and a function $h$ to make the diagram commute. For instance, the scientific method is an effective technique for observing an unknown function $m$ and formulating hypotheses on $\tau$ and $h$ that explain the behaviour of the function $m$. Note that although all functions can be interpretable, not all interpretable functions can be easily understood by all users (i.e., interpretability is relative to a user $h$).
    \item 
    \textbf{Interpretability is a spectrum.} If the diagram commutes for any possible $X^{(m)}$, then the function $m$ is completely interpretable by $h$. However, even if the diagram commutes only for a subset of $X^{(m)}$, the function $m$ can still be regarded as \textit{partially interpretable}. Thus, interpretability is best understood as a spectrum of degrees rather than an absolute, all-or-nothing property.
    \item 
    \textbf{Naively verifying interpretability via inference equivariance is intractable.}
    If we verify inference equivariance for a set of training samples, we do not guarantee that inference equivariance will hold for unseen samples. To guarantee this, we need to verify inference equivariance for any possible configuration of the inputs. However, this requires a table with $\mathcal{O}(\exp(D))$ entries. This means that if we consider $D = 10 \times 10$ binary pixels as features $X$, we already need more entries in the table than the number of atoms in the observable universe. 
    \item 
    \textbf{Many translations exist, but some are not sound.} In our example, we verified inference equivariance by translating $X_1(\inlineimg{zeroblue})=0$ to $\texttt{one} = \text{no}$, but translating $X_1(\inlineimg{zeroblue})=0$ to $\texttt{unum} = \text{no}$ could have also worked. However, the diagram would not commute if we translate both $X_1(\inlineimg{zeroblue})=0$ and $X_1(\inlineimg{oneblue})=1$ with $\texttt{one} = \text{yes}$.
\end{enumerate}
While the first observation gives us hope, the last two observations raise the following questions: under which assumptions is inference equivariance tractable? When is $\tau$ a sound translation? To answer these questions, we first identify assumptions and principles that make inference equivariance tractable, and then we characterise sound translations.
In~§\ref{sec:breaking} we show how standard assumptions in representation learning significantly simplify inference equivariance in terms of a set of variables $C$ that is much smaller than the set $X$. In~§\ref{sec:concept} we show that by properly characterising the variables~$C$, sound translations can be described as a by-product.

\section{Effective Equivariance Verification} \label{sec:breaking}

Circumventing the intractability of inference equivariance requires compressing the table representation of $m$ (e.g., Table~1). If this compression exists, then we can verify equivariance using a \textit{smaller} set of features $C$ characterising \emph{only the essential properties} of each object $\omega$. The following definition formalises compression properties in our context.
\begin{definition}(Lossless latent space)
Given a feature space $X \subseteq \mathbb{R}^D$ and a task space $Y \subseteq \mathbb{N}$, then $C \subseteq \mathbb{R}^K$ is a \textbf{lossless latent space} if $C$:
\begin{enumerate}
    \item represents $X$ in fewer dimensions ($K \ll D)$;
    \item preserves task-relevant information: $I(Y; C) \approx I(Y; X)$, where $I(\cdot; \cdot)$ denotes mutual information.
\end{enumerate}
\end{definition}
The assumption underlying this compression is often known as the \textit{manifold hypothesis} \citep{cayton2008algorithms}, a standard assumption for all representation learning systems~\citep{bengio2013representation}, including neural networks.

\paragraph{Conditional interpretability}
It may seem that by introducing latent variables $C$, one unnecessarily increases the size of Table~1. However, the following allows us to verify inference equivariance considering \textit{only} features in $C$: 

\begin{definition}(Conditional interpretability)
A variable $Y$ is \textbf{conditionally interpretable} given $\{C_i\}_{i=1}^K$ if
\[
I(Y; X_j \mid \{C_i\}_{i=1}^K) = 0 \quad \forall X_j \notin \{C_i\}_{i=1}^K.
\]
\end{definition}
The set $\{C_i\}_{i=1}^K$ is often known as a \textbf{Markov blanket} \citep{pearl1988probabilistic} of the variable $Y$ and will be denoted by $\mathcal{B}(Y)$. Intuitively, the definition means that once we know $\mathcal{B}(Y)$, any variable $X_j \notin \mathcal{B}(Y)$ does not provide additional information to explain $Y$. So, when verifying inference equivariance for a model $P(Y \mid \mathcal{B}(Y))$, we can ignore all $X_j \notin \mathcal{B}(Y)$.

\paragraph{[Consequence] Manifold-induced re-parametrisation}
Under the manifold assumption, we can use conditional interpretability to rewrite any model $P(Y \mid X)$ as follows:
\begin{equation} \label{eq:reparametrisation}
    P(Y, C, X) = P(Y \mid C) P(C \mid X), \quad \text{s.t.} \ C \coloneqq \mathcal{B}(Y) 
\end{equation}
This means that we can focus exclusively on the conditional distribution $P(Y \mid C)$ to explain the behaviour of $Y$.  As a result, we can rewrite any table representing a function $P(Y \mid X)$ based on a set of variables $C$, which is much smaller than the number of features $X$, thus reducing the table size from $\mathcal{O}(\exp(D))$ to $\mathcal{O}(\exp(K))$. 

\paragraph{[Consequence] Manifold-induced generalisation}
Lossless latent spaces not only reduce the number of columns in a function's tabular representation (e.g., Table~1), but also reduce the number of unique rows through an effect called \emph{generalisation}~\citep{kawaguchi2017generalization,neyshabur2017exploring}. In particular, when we compress information, objects having \emph{different} features $X$ may end up having the \emph{same} features $C$. This means that for any assignment to variables $C$, we can verify inference equivariance for multiple objects in one shot. In fact, equivariance would still hold for any object $\omega' \neq \omega$ as long as $C(\omega') = C(\omega)$. 

\begin{example}
    Verifying inference equivariance for $\omega_1$ (left diagram) does not guarantee that inference equivariance would hold for $\omega_2$ (right diagram) since $X(\omega_1) \neq X(\omega_2)$: 
    {\small
    \[
    \begin{tikzcd}[column sep=small, row sep=small]
    X^{(m)}(\omega_1) = \inlineimg{zerored} \arrow[d, "\tau"'] \arrow[r, "m"] & Y^{(m)}(\omega_1) \arrow[d, "\tau"] \\
    X^{(h)}(\omega_1) = \inlineimg{zerored} \arrow[r, "h"'] & Y^{(h)}(\omega_1)
    \end{tikzcd} 
    \
    \begin{tikzcd}[column sep=small, row sep=small]
    X^{(m)}(\omega_2) = \inlineimg{zeroblue} \arrow[d, "\tau"'] \arrow[r, "m"] & Y^{(m)}(\omega_2) \arrow[d, "\tau"] \\
    X^{(h)}(\omega_2) = \inlineimg{zeroblue} \arrow[r, "h"'] & Y^{(h)}(\omega_2)
    \end{tikzcd}
    \]
    }
    Conversely, if $C(\omega_1) = C(\omega_2)$ (e.g., both represented by the embedding $[-2.2, 1.3, 0.1] \in \mathbb{R}^3$), verifying inference equivariance on a compressed space $C$ for $\omega_1$:
    {\small
    \[
    \begin{tikzcd}[column sep=normal, row sep=small]
    C^{(m)}(\omega_1) = [-2.2, 1.3, 0.1] = C^{(m)}(\omega_2) \arrow[d, "\tau"'] \arrow[r, "m"] & Y^{(m)}(\omega_1) \arrow[d, "\tau"] \\
    C^{(h)}(\omega_1) \arrow[r, "h"'] & Y^{(h)}(\omega_1)
    \end{tikzcd}
    \]
    }
    guarantees that inference equivariance holds for $\omega_2$. Hence, if $m(C^{(m)}(\omega_1))$ is interpretable for the observer~$h$, then interpretability generalises also to $m(C^{(m)}(\omega_2))$.
\end{example}

\section{Concepts \& Sound Translations} \label{sec:concept}
Having motivated how the manifold hypothesis induces information to be compressed into a set of variables $C$, we show that by properly characterising variables $C$, we can get sound translations as a by-product. In this case, we will refer to variables $C$ as \textbf{concepts}.
We start by defining what a concept is, drawing from Formal Concept Analysis \citep{ganter1996formal} and Institution Theory \citep{goguen2005concept,diaconescu2008institution}. We extend this definition by providing a probabilistic interpretation of concepts, which corresponds to the informal notion commonly used in concept-based interpretability \citep{cbm,schut2025bridging}. Then, we show how sound translations \textit{preserve} concepts. Based on these insights, we recast our interpretability definition as \textit{concept-based inference equivariance}, a tractable formulation that enables the verification of translation soundness.


\subsection{What Is a ``Concept''?}
How can people communicate the notion of ``red''? Traditionally, we do this via two main ways: we can (1)~use a sequence of letters such as \texttt{red}, or (2)~refer to a concrete example such as \inlineimg{onered}. 
In a sense, communicating a ``concept'' requires that people agree on two implicit mappings: (1)~given the specific symbol \texttt{red}, we can associate it with concrete examples such as \inlineimg{onered}; and (2)~given an example such as \inlineimg{onered}, we can associate it with a symbol \texttt{red}.
As a result, we could give a first intuition of (1)~a \textbf{concept} as a relation between set of concrete examples (e.g., $\{\inlineimg{onered}, \inlineimg{zerored}, \inlineimg{appler}, \dots\}$) and a symbol (e.g., \texttt{red}); and (2)~a \textbf{sound translation} as a ``concept-preserving'' map associating different symbols (e.g., \texttt{red} and \texttt{rubrum}) to the same objects (e.g., $\{\inlineimg{onered}, \inlineimg{zerored}, \inlineimg{appler}, \dots\}$).
In what follows, we dive deeper into understanding what a concept entails through a concrete example. 
\begin{example}
Consider a set of sentences $S=\{ \texttt{red}, \texttt{one} \wedge \neg \texttt{fruit}, \texttt{zero}, \texttt{even} \}$ and a set of objects $U$ with the following relations with each sentence:
{\small 
\[
\begin{array}{c|ccc|c}
 & \texttt{red} & \texttt{one} \wedge \neg \texttt{fruit} & \texttt{zero} & \texttt{even} \\
\hline
\inlineimg{onered} & 1 & 1 & 0 & 0 \\
\inlineimg{zeroblue} & 0 & 0 & 1 & 1 \\
\inlineimg{zerored} & 1 & 0 & 1 & 1 \\
\inlineimg{oneblue} & 0 & 1 & 0 & 0 \\
\inlineimg{appler} & 1 & 0 & 0 & 0 \\
\end{array}
\]
}
 Consider a set of sentences $T=\{\texttt{red}\} \subseteq S$ and let $\beta$ be a function that gives us all objects $\omega \in M \subseteq U$ satisfying each $\phi \in T$ (which we traditionally denote as $\omega \models \phi$):
 \[
    M = \beta(T) = \beta(\{\texttt{red}\}) = \{ \inlineimg{appler}, \inlineimg{onered}, \inlineimg{zerored} \}
 \]
 If we consider a function $\gamma$ giving us the set of sentences that are true for all objects $M=\{ \inlineimg{appler}, \inlineimg{onered}, \inlineimg{zerored} \}$, this returns:
 \[
    \gamma(\beta(T)) = \gamma(\{ \inlineimg{appler}, \inlineimg{onered}, \inlineimg{zerored} \}) = \gamma(M) = \{\texttt{red}\} = T
 \]
 Note how the set of objects $M = \{ \inlineimg{appler}, \inlineimg{onered}, \inlineimg{zerored} \}$ and the sentence $T=\{\texttt{red}\}$ satisfy a specific ``closure'' condition:
\begin{align*}
    \resizebox{.28\textwidth}{!}{\exampleConceptRound}
\end{align*}
 \[
    T = \gamma(M) \quad \text{and} \quad M = \beta(T)
 \]
 \begin{center}
\text{\textbf{Figure 1}: Closure between objects $M$ and sentences $T$.}
\end{center}

 \noindent
 Hence, we can refer to the concept ``red'' as the tuple $(\{\texttt{red}\}, \{ \inlineimg{appler}, \inlineimg{onered}, \inlineimg{zerored} \}, \beta, \gamma)$.
 Note how (1)~this closure is not satisfied by the objects $M = \{ \inlineimg{appler}, \inlineimg{onered}, \inlineimg{zerored}, \inlineimg{zeroblue} \}$ since \inlineimg{zeroblue} does not satisfy the sentence \texttt{red}; (2)~if we add a sentence to $T$, we end up with a more specific concept since fewer objects satisfy all sentences: $(\{\texttt{red}, \texttt{zero}\}, \{ \inlineimg{zerored} \}, \beta, \gamma)$.
\end{example}

Following~\citet{goguen2005concept}, we formalise a concept via a set of objects $U$, a set of sentences $S$, and two functions:
\begin{itemize}[topsep=3pt, leftmargin=9pt, itemsep=0pt]
    \item  $\beta : \mathcal{P}(S) \to \mathcal{P}(U)$ is a function producing the set of all objects $\omega$ that satisfy every sentence $\varphi$ in $T \subseteq S$ (where $\mathcal{P}(A)$ denotes the power set of $A$):
        \[
        \beta(T) = \{ \omega \in U \; \mid \; \omega \models \varphi \text{ for all } \varphi \in T \},
        \]
    \item $\gamma : \mathcal{P}(U) \to \mathcal{P}(S)$ is a function producing the set of all sentences satisfied by every object in $M \subseteq U$:     
        \[
        \gamma(M) = \{ \varphi \in S \; \mid \; \omega \models \varphi \text{ for all } \omega \in M \}.
        \]
\end{itemize}

\begin{definition}[Concept]
Given a set of objects $U$ and a set of sentences $S$, a \textbf{concept} is a tuple $(T \subseteq S, M \subseteq U, \gamma, \beta)$ such that $T = \gamma(M) \text{ and } M = \beta(T)$.
\end{definition}

\subsection{Probabilistic Interpretation of Concepts}
We can extend the definition above by providing a probabilistic interpretation of concepts and demonstrating how it aligns with commonly accepted notions in the concept-based interpretability literature~\citep{tcav,cbm}.

If we allow uncertainty over objects, the random variable $X: \Omega \to \mathbb{R}^D$ describes the features of an object drawn from this unknown distribution. We can interpret \textbf{concept membership} $C_i$ as an indicator random variable of the event ``the object belongs to the set of objects $M_i$ of the $i$-th concept'':
\[
C_i = \mathbb{I}_{X(\omega) \in M_i} =
\begin{cases}
    1, \quad \text{if $X(\omega)$ belongs to $M_i$}\\
    0, \quad \text{otherwise.}\\
\end{cases}
\]
\begin{example}
The random object $X(\omega) = \inlineimg{onered}$ belongs to the concept ``red'' since $\inlineimg{onered} \in \{\inlineimg{appler}, \inlineimg{onered}, \inlineimg{zerored}\} = M_{\texttt{red}}$. This makes the concept membership\footnote{To improve readability in examples, we abuse notation and use strings for subscripts instead of natural numbers.} $C_{\texttt{red}} = \mathbb{I}_{\inlineimgsmall{onered} \in M_{\texttt{red}}} = 1$.
\end{example}
If concept membership is not given but rather uncertain, the indicator function becomes a probability function:
\[
g_i: X \to [0,1]
\]
where $g_i(x)$ is the probability that $x$ belongs to the $i$-th concept. For any $x = X(\omega)$, the membership indicator $C_i$ is reduced to a Bernoulli random variable with parameter $g_i$: 
\[
    P(C_i=1 \mid X=x) = g_i(x) 
\]
This corresponds to standard notions of ``concepts'' in general concept-based interpretable models such as \textit{Concept Bottleneck Models} (CBMs)~\cite{cbm}.
\begin{example}
Suppose that membership in ``red'' of $X(\omega)=\inlineimg{onered}$ is uncertain, then $g_{\texttt{red}}$ gives us the probability that the object is red: $g_{\texttt{red}}(\inlineimg{onered}) = P(C_{\texttt{red}} = 1 \mid X = \inlineimg{onered}) = 0.9$.
\end{example}
We can easily extend this definition to accommodate diverse concept distributions. For example, the concept ``digit'' in MNIST has a categorical distribution, while the concept ``red intensity'' may have a Beta distribution.

\subsection{When Is a Translation Sound?}
We now show how concepts allow the characterisation of sound translations. In particular, concepts emphasise that (1)~translations are functions between (purely syntactic) sentences, (2)~translations induce concept transformations $C_i \to C_{\tau(i)}$, and (3)~sound translations must ``preserve concepts'', that is, if an object satisfies a sentence $\varphi$, it should also satisfy the translated sentence $\tau(\varphi)$. We refer to such sound translations as \textit{concept-based translations} $\tau_c$.

\begin{definition}(Concept-based translation)
Given a pair of concepts $C = (T, M, \gamma, \beta)$ and $C' = (T', M', \gamma', \beta')$, a \textbf{sentence translation function $\tau_c : T \to T'$ is sound} if it preserves \emph{concept closure} on the same set of objects $M^* \neq \emptyset$:
{\small 
\[
    \begin{tikzcd}[column sep=normal, row sep=small]
    M^* \arrow[r, "\gamma"] \arrow[rd, "\gamma'"'] & T \arrow[rd, "\beta"] \arrow[d, "\tau_c"] & \\
    & T' \arrow[r, "\beta'"']  & M^*
    \end{tikzcd}
\]
}
\end{definition}

\begin{example}
Given the sentences $T = \{\texttt{red}\}$ and $T' = \{\texttt{rubrum}, \texttt{unum}\}$, the objects $M^* = \{\inlineimg{appler}, \inlineimg{onered}\}$, the translation $\tau_c = \{\texttt{red} \to \texttt{rubrum}\}$ is sound as it preserves concept closure, while $\tau = \{\texttt{red} \to \texttt{unum}\}$ is \textbf{not} sound as it does not preserve closure:
{\scriptsize 
\[
    \begin{tikzcd}[column sep=small, row sep=normal]
    \{\inlineimg{appler}, \inlineimg{onered}\} \arrow[r, "\gamma"] \arrow[rd, "\gamma'"'] & \{\texttt{red}\} \arrow[d, "\tau_{c}"] \arrow[dr, "\beta"] & \\
    & \{\texttt{rubrum}\} \arrow[r, "\beta'"] & \{\inlineimg{appler}, \inlineimg{onered}\}
    \end{tikzcd}
    \begin{tikzcd}[column sep=small, row sep=normal]
    \{\inlineimg{appler}, \inlineimg{onered}\} \arrow[r, "\gamma"] & \{\texttt{red}\} \arrow[d, "\tau"] \arrow[r, "\beta"] & \{\inlineimg{appler}, \inlineimg{onered}\} \\
    \{\inlineimg{onered}, \inlineimg{oneblue}\} \arrow[r, "\gamma'"]
    & \{\texttt{unum}\} \arrow[r, "\beta'"] & \{\inlineimg{onered}, \inlineimg{oneblue}\}
    \end{tikzcd}
\]
}
\end{example}
To find sound translations in practice, we typically minimise the divergence between a given concept distribution $C$ and a reference distribution $C^{(h)}$~\citep{cbm}.

\section{Tractable \& Sound Inference Equivariance} \label{sec:concept-based-eq}
We can now provide an important result showing that the tools we introduced in §\ref{sec:breaking}-\ref{sec:concept} (i.e., conditional interpretability, lossless latent spaces, and sound translations) are necessary and sufficient to bound the number of steps required to verify interpretability (see proof in App.~§\ref{app:proofs}). 
\begin{theorem}(Bounded verification of interpretability)
    Given a task $Y$ and a feature space $X \subseteq \mathbb{R}^D$, inference equivariance is verifiable in $L < \exp(D)$ steps \emph{iff} the task is conditionally interpretable given a lossless latent space $C \subseteq \mathbb{N}^K$ such that: (a)~$K < D$, and (b)~$\tau$ is a sound translation for all $C_i$ and task $Y$.
\end{theorem}
Based on this result, we can recast our interpretability test as a concept-based inference equivariance.
\begin{definition}(Concept-based inference equivariance) 
Given a pair of concept probability functions $g$ and $g'$, a pair of \textit{task predictor} functions $f: C_1 \times \dots \times C_{K_1} \to Y$ and $f': C_1' \times \dots \times C_{K_2}' \to Y'$, and a concept-based translation function $\tau_c : T \to T'$, the two functions $f$ and $f'$ satisfy \textbf{concept-based inference equivariance} if the following diagram commutes $\forall X$:
{\small 
\[
    \begin{tikzcd}[column sep=normal, row sep=normal]
    X \arrow[r, "g"] \arrow[rd, "g'"'] & \{C_i\}_{i=1}^{K_1} \arrow[r, "f"] \arrow[d, "\tau_c"] & Y \arrow[d, "\tau_c"'] \\
    & \{C_j'\}_{j=1}^{K_2} \arrow[r, "f'"']  & Y'
    \end{tikzcd}
\]
}
\end{definition}

\begin{example}
Given an object $\inlineimg{zerored} \in X$, sentences $T_C = \{\texttt{one}, \texttt{red}\}$, $T_C' = \{\texttt{unum}, \texttt{rubrum}\}$, derived sentences $T_Y = \{\texttt{even}\}$, $T_Y' = \{\texttt{par}\}$, and a English-to-Latin translator $\tau_c$, this diagram commutes:
{\small 
\[
    \begin{adjustbox}{max width=\linewidth}
    \begin{tikzcd}[column sep=normal, row sep=small]
    \inlineimg{zerored} \arrow[r, "g"] \arrow[rd, "g'"'] & \{C_{\texttt{one}}=0, C_{\texttt{red}}=1\} \arrow[r, "f"] \arrow[d, "\tau_c"] & \{Y_{\texttt{even}} = 1\} \arrow[d, "\tau_c"'] \\
    & \{C'_{\texttt{unum}}=0, C'_{\texttt{rubrum}}=1\} \arrow[r, "f'"']  & \{Y'_{\texttt{par}} = 1\}
    \end{tikzcd}
    \end{adjustbox}
\]
}
In this example, verifying concept-based inference equivariance requires three checks ($C_{\texttt{unum}} \stackrel{?}{=} \tau_c(C'_{\texttt{one}})$,  $C_{\texttt{rubrum}} \stackrel{?}{=} \tau_c(C'_{\texttt{red}})$, and $C_{\texttt{par}} \stackrel{?}{=} \tau_c(C'_{\texttt{even}})$) to guarantee equivariance for any example with the same concept representation. In contrast, pixel-space verification requires $32\times 32$ checks and applies only to objects with identical pixel representations.
\end{example}

\noindent This has three key advantages over Definition~\ref{def:infeq}:
\begin{enumerate}
    \item 
    \textbf{Scalability:} Under the manifold hypothesis, and thanks to conditional interpretability, the size of the table we need to build to verify inference equivariance for $P(Y \mid C)$ is exponentially smaller than for $P(Y \mid X)$ ($\mathcal{O}(\text{exp}(K))$  rather than $\mathcal{O}(\text{exp}(D))$), and can be reduced even further as we show in §\ref{sec:pyc}.
    \item 
    \textbf{Sound translation:} Concept structures enable a precise characterisation of sound translations $\tau_c$ as syntactic mappings, which preserve a concept's closure.
    \item 
    \textbf{Generalisation:} The compression induced by the manifold hypothesis encourages the representations of similar objects to collapse, enabling the verification of inference equivariance on a single object to be extended to any object sharing the same concept representation.
\end{enumerate}

\section{Consequences on Architectural Design} \label{sec:principles}
This section discusses how the assumptions introduced thus far impact model design by answering the following questions: How can $P(C\mid X)$ compress information, effectively discarding irrelevant details while preserving relevant information (§\ref{sec:pcx})? How can $P(Y \mid C)$ further reduce the number of steps required to verify inference equivariance (§\ref{sec:pyc})? What role do the parameters $\Theta$ play in determining the expressivity and interpretability of a parametric model $P(Y \mid C; \Theta)$ (§\ref{sec:parametrisation})? How can humans effectively interact and align with an interpretable model (§\ref{sec:alignment})?

\subsection{Design Considerations for $P(C \mid X)$} \label{sec:pcx}
How can we discard irrelevant information and retain useful information in concept representations in order to generate a compact but informative lossless concept space? 

\paragraph{Concept invariance discards irrelevant information}
Concept invariances enable us to ignore irrelevant variations -- for example, a rotated zero remains a zero.
Following \citet{bronstein2021geometric}, we formalise invariances by introducing $\mathfrak{G}$ as a group acting on the input space $X$ via the group action $\mathfrak{b} \cdot x$ for $\mathfrak{b} \in \mathfrak{G}$ and $x = X(\omega)$. 
We consider for each group action $\mathfrak{b}$ on $X$, a corresponding action on the concept space $C$, $\rho: \mathfrak{G} \to \mathrm{Aut}(C)$ where $\mathrm{Aut}(C)$ is the group of automorphisms of $C$, that is, structure-preserving bijections $C \to C$. In other words, the map $\rho$ associates to each $\mathfrak{b} \in \mathfrak{G}$ a transformation $\rho(\mathfrak{b}): C \to C$ describing how the concept labels change under the group action $\mathfrak{b}$. 
\begin{definition}(Concept invariance)
The function $g: X \rightarrow C$ is \textbf{concept invariant} w.r.t.\ group action $\mathfrak{b}$ on $X$ if, $\forall \mathfrak{b} \in \mathfrak{G}$ and $\forall x_i \in X$ s.t. $\mathfrak{b}(x_i) = x_j$,  the following diagram commutes:
{\small 
\[
\begin{tikzcd}[column sep=normal, row sep=normal]
x_i \arrow[r, "g"] \arrow[d, "\mathfrak{b}"'] & C \arrow[d, "\text{id}"] \\
x_j \arrow[r, "g"']  & C
\end{tikzcd}
\]
}
\end{definition}
\begin{example}
Given an image \inlineimg{zerored}, the group action that rotates an image should not impact the concept ``red'':
{\small
\[
\begin{tikzcd}[column sep=normal, row sep=small]
\inlineimg{zerored} \arrow[r, "g"] \arrow[d, "\mathfrak{b}"', shorten >=5pt] & C_\texttt{red}=1 \arrow[d, "\text{id}"] \\
\inlineimgrotate{zerored} \arrow[r, "g"']  & C_\texttt{red}=1
\end{tikzcd}
\]
}
\end{example}
Invariances could be structural (embedded in the architecture as convolution) or operational (as data augmentations).

\paragraph{Concept equivariance preserves useful information}
While invariances allow for the discarding of information, concept equivariances preserve information from $X$.
\begin{definition}(Concept equivariance)
The function $g: X \rightarrow C$ is \textbf{concept equivariant} w.r.t.\ group actions $\mathfrak{b}$ on $X$ and $\rho(\mathfrak{b})$ on $C$ if, $\forall \mathfrak{b} \in \mathfrak{G}$ and $\forall x_i \in X$ s.t. $\mathfrak{b}(x_i) = x_j$, the following diagram commutes:
{\small 
\[
\begin{tikzcd}[column sep=normal, row sep=normal]
x_i \arrow[r, "g"] \arrow[d, "\mathfrak{b}"'] & C \arrow[d, "\rho(\mathfrak{b})"] \\
x_j \arrow[r, "g"']  & C'
\end{tikzcd}
\]
}
\end{definition}
\begin{example}
Given (1) an image \inlineimg{zerored}, (2) a pixel-space group action $\mathfrak{b}$ that changes the background colour to ``blue'', and (3) a concept-level group action $\rho(\mathfrak{b})$ that sets all non-blue concepts $C_i$ to $0$ while setting $C_\text{blue} := 1$, a function $g$ that accurately predicts the background colour in $X$ is concept equivariant given $\mathfrak{b}$ and $\rho(\mathfrak{b})$ as this diagram commutes:
{\small 
\[
\begin{tikzcd}[column sep=normal, row sep=small]
\inlineimg{zerored} \arrow[r, "g"] \arrow[d, "\mathfrak{b}"', shorten >=5pt] & C_\texttt{red}=1 \arrow[d, "\rho(\mathfrak{b})"] \\
\inlineimg{zeroblue} \arrow[r, "g"']  & C_\texttt{red}=0
\end{tikzcd}
\]
}
\end{example}
We discuss spurious invariances and equivariances in §\ref{app:leakage}.


\subsection{Design Considerations for $P(Y \mid C)$} \label{sec:pyc}

$P(Y \mid C)$ is ideally a function that further simplifies inference equivariance verification. Here, we show how compositionality and sparsity can contribute to this objective.

\paragraph{Compositionality splits functions into simpler parts}
Compositionality enables us to rewrite a model $P(Y \mid C)$ as a composition of simpler models~\citep{fong2018seven,coecke2018picturing,elmoznino2024complexity,hewittposition}. The core idea is to use a finite set of elementary ``processes'' -- that is, simple, basic functions -- to build more complex functions~\citep{hewittposition}, similarly to how we use a finite vocabulary to formulate an infinite number of sentences in human languages~\citep{chomsky1957syntactic}. Following~\citet{lorenz2023causal}, we use network diagrams (NDs), sound and complete ways of formalising probabilistic and causal process\footnote{Any model isomorphic to a ND works. Yet, NDs generalise factor graphs and probabilistic graphical models)~\citep{forney2002codes}.}, to describe \textit{concept-based processes}:

\begin{definition}(Concept-based process)
A concept-based process is a diagram built from \emph{single output boxes} (which transform input concepts into other concepts), \emph{copy maps} (which duplicate concepts), \emph{discarding effects} (which discard concepts), and \emph{constants}:
\[
    \resizebox{.68\textwidth}{!}{\compositionalElements}
\]
\end{definition}
Probabilistically, we interpret boxes without input as probability distributions, and boxes with inputs as functions between distributions.
For example, by composing boxes, we can rewrite the following 3-input process $P(Y \mid C)$ as a composition of 2-input processes:
\[
    \resizebox{.45\textwidth}{!}{\compositionalExample}
\]

\paragraph{Sparsity prunes a function's inputs}
The size of the table describing the function $f_i$ of a single output box producing $C_i$ depends on the number of input concepts $\text{pa}(C_i)$ (a.k.a. ``parents''). By enforcing sparsity on the input set, we can eliminate redundant input concepts, simplifying elementary processes -- as a plethora of previous works have emphasised~\citep{aristotleposterior,punch15johannes,miller1956magical,kolmogorov1965three,rissanen1978modeling,schmidt2009distilling,rudin2019stop,goldblum2023no} -- and thus making the verification of inference equivariance more efficient.
\begin{definition}(Sparse concept-based process)
A process $C_i = f_i(\text{pa}(C_i))$ is sparse if $|\text{pa}(C_i)| \ll |C|$, where $\text{pa}(C_i)$ is the set of parent notes for $C_i$ (i.e., its ``inputs'') and $K$ is the number of total concepts.
\end{definition}

\subsection{Design Considerations for $P(Y \mid C; \Theta)$ Parametrisation} \label{sec:parametrisation}
If the probability distribution of a given task $P(Y \mid C; \Theta)$ depends on a set of parameters $\theta \in \Theta$, then the Markov blanket of $Y$ includes both concepts $C$ and parameters $\Theta$, that is, $\mathcal{B}(Y) = C \cup \Theta$ (further discussion in §\ref{app:transparencies}).
We can then rewrite the manifold-induced re-parametrisation of the joint distribution $P(Y, C, X; \Theta)$ as:
\[
    \resizebox{.2\textwidth}{!}{\expressivityParametricJoint}
\]

\paragraph{Maximise expressivity while preserving interpretability}
The above re-parametrisation emphasises two potential limitations for the overall expressivity:
\begin{enumerate}
    \item \textbf{Incompleteness limits expressivity:} Depending on the task and data availability, constructing lossless concept latent spaces is not trivial. Unfortunately, whenever we have $I(Y;C) < I(Y;X)$, we end up with a ``concept bottleneck'', which limits expressivity due to a loss of information in the concept latent space~\citep{concept_completeness,promises_and_pitfalls,cem}.
    \item \textbf{Sparsity limits expressivity:} While sparsity prompts concept processes to prune input concepts, over-pruning can inadvertently remove concepts holding useful information for the downstream task, thus further reducing expressivity~\citep{arrieta2020explainable}.
\end{enumerate}
A workaround to maximise expressivity while preserving interpretability -- and relax the assumption that $C$ is lossless -- is to neurally re-parametrise $\Theta$, making the parameters input-dependent\footnote{Note that input-dependent parametrisations subsume input-independent parametrisations when $P(\Theta \mid X)$ is constant $\forall X$.}~\citep{senn,barbiero2023interpretable}:
\[
    \resizebox{.2\textwidth}{!}{\expressivityDCR}
\]

\paragraph{Concept memory enables verifiability}
Using input-dependent parameters makes the behaviour of concept-based processes unpredictable on unseen data, as parameters $\Theta$ are unknown a priori. This means that we cannot easily verify the behaviour of these processes. To enable verifiability, we can introduce an input-dependent selection $R$ over a fixed-size ``\textit{memory}'' of parameters $Q_\Theta$~\citep{debot2024interpretable}:
\[
    \resizebox{.25\textwidth}{!}{\expressivityCMR}
\]
This way, the possible parameter states are finite and verifiable, but the choice within this finite set is input-specific.

\subsection{Human-Machine Interaction and Alignment} \label{sec:alignment}
\paragraph{Concepts enable human interventions}
A key advantage of concept-based models is their support for human interaction. Users can \textit{intervene} on concept predictions~\cite{cbm, coop, barker2023selective, closer_look_at_interventions, collins2023human, intcem, beyond_cbms}, adjust parameters of $P(Y \mid C;\Theta)$~\cite{yuksekgonul2022post, barbiero2023interpretable, debot2024interpretable, barbiero2024relational}, or re-wire the dependencies between concepts and tasks~\cite{stochastic_cbms,dominici2024causal,debot2025interpretable}. 
Two typical types of interventions are ground-truth and do-interventions. Ground-truth interventions \raisebox{-5pt}{\smash{\resizebox{.07\textwidth}{!}{\gtInterventionSmall}}} (Eq.~\ref{eq:interventions}, left) replace a concept's distribution $P(C_i \mid X)$ with a ground-truth distribution $C_i$. This way, we can fix errors introduced by $P(C_i \mid X)$ and improve the task accuracy. \mbox{Do-interventions} \raisebox{-4pt}{\smash{\resizebox{.06\textwidth}{!}{\doInterventionSmall}}} (Eq.~\ref{eq:interventions}, right) replace a concept's distribution with a constant value $k$~\citep{pearl2016causal} and can be used to estimate the average causal effect of a concept on a downstream task \citep{cace}.
\begin{equation*} \label{eq:interventions}
    \resizebox{.23\textwidth}{!}{\gtinterventionExample} \ \resizebox{.23\textwidth}{!}{\dointerventionExample}
\end{equation*}

\paragraph{Alignment enables concept identifiability}
Which translation should a model learn when multiple sound translations exist? For instance, suppose that $\tau_c: \{\texttt{nulla} \to \texttt{one}, \texttt{unum} \to \texttt{zero}, \texttt{par} \to \texttt{even}\}$ is sound and that the following diagram commutes:
{\small
\[
    \begin{adjustbox}{max width=\linewidth}
    \begin{tikzcd}[column sep=normal, row sep=small]
    \inlineimg{zerored} \arrow[r, "g"] \arrow[rd, "g'"'] & \{C_{\texttt{nulla}}=0, C_{\texttt{unum}}=1\} \arrow[r, "f"] \arrow[d, "\tau_c"] & Y_{\texttt{par}} = 1 \arrow[d, "\tau_c"] \\
    & \{ C'_{\texttt{one}}=0, C'_{\texttt{zero}}=1\} \arrow[r, "f'"']  & Y'_{\texttt{even}} = 1
    \end{tikzcd}
    \end{adjustbox}
\]
}
While the diagram commutes, we note that ``Latin'' concepts have the opposite meaning of the corresponding ``English'' concepts. This phenomenon, known as a \emph{reasoning shortcut}~\citep{geirhos2020shortcut,marconato2023neuro,marconato2024not,chollet2024arc},
arises when the data and model admit multiple indistinguishable concept assignments and sound translations. When this happens, \emph{aligned translations are not identifiable}~\citep{melsa1971system} without additional information. 
In such cases, an \emph{alignment mechanism} is required to select a sound translations from a distribution $P(\tau)$. Ante-hoc alignment methods address this by training the model $P(Y \mid C_\tau)$ conditioned on a fixed translation $\tau$~\citep{cbm,cem,probabilistic_cbms,glancenets,debot2024interpretable,dominici2024causal}, while post-hoc alignment methods search for a sound and aligned translation after training using probing techniques~\citep{alain2016understanding,ettinger2016probing,shi2016does,hewitt2019structural,burns2022discovering,ouyang2022training,zou2023representation,marks2023geometry,oikarinenlabel} -- as when using sparse autoencoders on a language model's neurons~\citep{cunningham2023sparse,templeton2024transformer}. 
In our example, we can select a translator to align ``Latin'' concepts with the closest matching ``English'' concepts using a probe to match $C_\texttt{unum}$ with $C'_\texttt{zero}$ and $C_\texttt{nulla}$ with $C'_\texttt{one}$, and then we re-label Latin concepts as $C_\texttt{unum} \to C_\texttt{nulla}$ and $C_\texttt{nulla} \to C_\texttt{unum}$. 

\section{Blueprint for Interpretable Models}
\label{sec:blueprint}

Based on the foundational properties discussed in previous sections, we can now outline the general structure of a concept-based interpretable model.
\begin{definition}(Blueprint for interpretable models)
    Under the \emph{manifold hypothesis} assumption, the \emph{conditional interpretability principle} allows to rewrite any model $P(Y \mid X)$ as an \textbf{interpretable model} 
    \[
    \resizebox{1\textwidth}{!}{\blueprint}
    \]
    where:
    \begin{itemize}
    \item $P(C, \Theta \mid X)$ is a compression process combining \emph{concept-based invariances} to discard irrelevant information and \emph{equivariances} to retain useful information such that $I(Y;X) \approx I(Y;C)$.
    \item $P(C_\tau, \Theta_\tau \mid C, \Theta, \tau)$ is an \emph{alignment mechanism} applying a \emph{sound translation} sampled from $P(\tau)$.
    \item $P(Y \mid C_\tau \ ; \ \Theta_\tau)$ is a \emph{compositional} and \emph{sparse} process where $\Theta_\tau$ are the parameters of the decision mechanism predicting the objective $Y$.
\end{itemize}
\end{definition}
The proposed blueprint informs researchers about the key ingredients for building interpretable models. To support the implementation of existing models and the development of novel models, we designed a Python library with native support for concept-based data structures and processes (§\ref{app:library}).

\section{Limitations \& Discussion}
This work brings together insights from a variety of research fields -- including representation learning, group theory, causality, institution theory, category theory, and social sciences -- to propose a formal, actionable definition of AI interpretability. This definition, though not universal, is straightforward, encompasses existing informal notions, and is contextualised within AI, allowing us to pinpoint the fundamental assumptions and principles behind interpretable models. To achieve this we use formalisms from different communities (e.g., commutative and network diagrams) which might introduce an overhead for readers unfamiliar with these notations. However, we aimed to strike a balance between an expressive, yet intuitive approach (e.g., allowing us to distinguish different types of interventions) to demonstrate how the core assumptions we identified directly influence model design. Building on these insights, we propose a blueprint for interpretable models and introduce a library for their implementation.  In essence, this work frames AI interpretability as a well-posed problem, sets forth enduring principles for building interpretable models, and introduces a theoretical framework which could be extended and used to identify new research directions, like determining suitable translations to establish interpretability equivalence between different models.

\section*{Acknowledgements}
The \texttt{PyC} library has been developed -- and continues to be refined -- by an exceptional team of collaborators, including Gabriele Ciravegna, David Debot, Michelangelo Diligenti, Gabriele Dominici, Francesco Giannini, Sonia Laguna, and Moritz Vandenhirtz. We also extend our gratitude to Alberto Tonda for his insightful feedback and fresh perspective on the drafts of this work. 

\section*{Disclosure of Funding}
PB acknowledges support from the Swiss National Science Foundation project IMAGINE (No. 224226). MEZ acknowledges support from the Gates Cambridge Trust via a Gates Cambridge Scholarship. GM acknowledges support from the KU Leuven Research Fund (STG/22/021, CELSA/24/008) and from the Flemish Government under the ``Onderzoeksprogramma Artificiële Intelligentie (AI) Vlaanderen'' programme. AT acknowledges support from the Hasler Foundation grant Malescamo (No. 22050), and the Horizon Europe grant Automotif (No. 101147693).



\bibliography{aaai2026}

\clearpage
\newpage
\appendix
\onecolumn



\section{Comparison With Selection of Existing Definitions of Interpretability} \label{app:other-def}
In this appendix, we discuss (1) how our definition of interpretability subsumes existing \textit{informal} definitions of interpretability proposed in the interpretable AI community and (2) how it relates to existing definitions in other fields. To this end, we compare our definition against a selection of the most cited and influential definitions of interpretability. 
Before diving into this discussion, we would like to remark that our definition of interpretability as inference equivariance arises from an existing and unpublished pre-print co-authored by PB, MEZ, and GM~\citep{barbiero2025neural}. However, here the definition has been significantly refined and represents only the starting point of deeper discussions we bring forth in the rest of the paper.

\subsection{Relation with Informal Definitions of Interpretability}
\paragraph{Definition by \citet{kim2016examples}} 
\citet{kim2016examples} suggested that \emph{a method is interpretable if a user can correctly and efficiently predict the method's results}. Inference equivariance formally captures this notion: the diagram commutes if the human user $h$ can achieve the same results as the model $m$ given the same input. Our definition, however, uniquely stresses the importance of clearly defining and characterising the translation function that maps knowledge from the model $m$ to the user $h$.

\paragraph{Definition by \citet{biran2017explanation}}
\citet{biran2017explanation} suggested that \emph{systems are interpretable if their operations can be understood by a human}. This ``understanding'' can be broken down into two aspects. If it refers to comprehending the direct mapping from input to output -- essentially, how the function works in a tabular sense -- then this concept is formalised by inference equivariance. However, if ``understanding operations'' means discerning how the model parameters influence the decision-making process, then as discussed in §\ref{sec:parametrisation}, these parameters fall within the Markov blanket of the target variable $Y$, and -- similarly to concepts -- inference equivariance is a way to formalise the understanding of the role of parameters in the decision-making process.

\paragraph{Definition by \citet{miller2019explanation}}
\citet{miller2019explanation} defines interpretability as \emph{the degree to which an observer can understand the cause of a decision}. This definition closely aligns with that of~\citet{kim2016examples}, allowing for similar reasoning. The main difference is that Miller's definition emphasises the causal aspect. In this regard, note that the Markov blanket of a target variable $Y$ encompasses by definition all its direct causes. Specifically, for a classification model $P(Y,C,X;\Theta) = P(Y \mid C; \Theta) P(C \mid X)$, the Markov blanket $\mathcal{B}(Y) \coloneqq C \cup \Theta$ comprises all (and only) causes of $Y$. As a result, by verifying concept-based inference equivariance (including the parameters $\Theta$), we can understand the relationship between $C$ and $\Theta$ -- the ``causes'' -- and $Y$ -- the decision. 

\subsection{Relation To Formal Definitions in Related Fields}
Following~\cite{rubenstein2017causal} and~\cite{geiger2024causal}, \citet{marconato2023interpretability} discuss in the context of interpretable AI the notion of \emph{causal abstractions}, that is, commutative diagrams describing interventional equivariance between two structural causal models. While causal abstractions have not been proposed as a formalisation of interpretability, our definition of inference equivariance draws inspiration from these works. However, our construction requires fewer assumptions, as it does not necessitate the full causality formalism (e.g., structural causal models) and its inherent assumptions (e.g., access to generative factors of variation). Our formulation might even generalise interventional equivariance, as interventions could be viewed as a form of inference on probabilistic models.\\

In contrast, \citet{tarski1953undecidable} define interpretability in the context of formal logic. They do so as follows: \emph{a theory T is interpretable in a theory S if and only if there exists a translation from the language of T into the language of S such that every theorem of T is translated into a theorem of S}. Our formulation is specifically inspired by this definition, particularly concerning the notion of translation, and can be considered a special case. The main advantages of our formulation are two-fold: (1) we have contextualised the definition specifically within the domain of interpretable AI, and (2) we leverage this definition to derive practical consequences relevant to ongoing interpretable AI research.

\section{Proofs} \label{app:proofs}
\setcounter{theorem}{0}

Below, we describe a very simple proof of Theorem 1 in Section~\ref{sec:concept-based-eq} of this paper.

\begin{theorem}
    Given a task $Y$ and a feature space $X \subseteq \mathbb{R}^D$, inference equivariance is verifiable in $L < \exp(D)$ steps \emph{iff} the task is conditionally interpretable given a lossless latent space $C \subseteq \mathbb{N}^K$ such that (a) $K < D$, and (b) $\tau$ is a sound translation for all $C_i$ and task $Y$.
\end{theorem}
\begin{proof}
    We want to prove that a set of conditions $A_i$ is necessary and sufficient for a property $B$. We will first prove necessity ($\bigwedge_i A_i \implies B$) and then sufficiency ($B \implies \bigwedge_i A_i$).\\

    \paragraph{Proof of necessity (assuming $\bigwedge_i A_i$).}
    Assume we are given: ($A_1$) a lossless latent space $C \subseteq \mathbb{N}^K$ of dimension $K < D$, ($A_2$)  task $Y$ that is conditionally interpretable by $C$, ($A_3$) a translation $\tau$ that is sound for all $C_i$ and task $Y$. We show that inference equivariance is verifiable in less than $\text{exp}(D)$ steps. By definition, conditional interpretability implies that the task $Y$ depends only on variables $C$. As a result, we do not have to consider features $X$ to verify inference equivariance. The sound translation guarantees closure for all concepts and tasks, so all variables can be interpreted individually. We can now count the number of steps we need to perform to verify inference equivariance between $C_i$ and $Y$. At most, we need $L = |\mathcal{P}(\{1, 2, \cdots, K\})| = 2^K < \exp(K)$ steps (as, in the worst-case scenario, one needs to generate all $2^K$ concept profiles). As we assumed $K < D$, this implies that the number of steps $L$ must be $L< \exp(K) < \exp(D)$, which is what we wanted to show.
    
    \paragraph{Proof of sufficiency by contradiction (assuming $B \wedge \neg \bigwedge_i A_i$).} 
    Assume that: ($B$) inference equivariance can be verified in $L < \exp(D)$ steps and ($\neg A_1$) $K \geq D$. Assuming conditional interpretability and that $\tau$ is a sound translation, we need $L = \exp(K)$ steps to verify the tabular representation of any function mapping $X$ to $Y$. However, $L = \exp(K) \geq \exp(D) > L$, which violates our assumption ($B$). Similarly, if we assume that the task is not conditionally interpretable ($\neg A_2$), we end up with even more (i.e., $\mathcal{O}(\exp(K + D))$) steps. Alternatively, if we assume that translations are not sound  ($\neg A_3$), we cannot even interpret variables individually.
\end{proof}



\section{Leakage} \label{app:leakage}
A big role in concept encoders is played by leakage, which could be both a curse (for interpretability) and a blessing (for expressivity). There are two main types of leakage:
\begin{itemize}
    \item \textbf{Task leakage}: This happens when information from $X$ could further explain $Y$ beyond $C$ i.e., when $I(Y; X \mid C) > 0$. A model $P(Y_j, C \mid X)$ is subject to \emph{task leakage} with respect to the group actions $\mathfrak{b}$ on $X$ and $\rho(\mathfrak{b})$ on $C$ if:
    \begin{equation*}
    \begin{adjustbox}{max width=\linewidth}
    $
    \displaystyle
        \exists x \in X, \quad \exists \mathfrak{b} \in \mathfrak{G}, \quad \exists \rho_j(\mathfrak{b}): Y \to Y, \quad 
        P(Y_j, C \mid \mathfrak{b} \cdot x) = P(\rho_j(\mathfrak{b}) \cdot Y_j \mid \text{id}_C \cdot  C) P(\text{id}_C \cdot C \mid x).
    $
    \end{adjustbox}
    \end{equation*}
    For instance, given an image \inlineimg{appler}, changing pixel intensities does not change the concept ``red'', but changes the object type:
    \begin{equation*}
    \begin{adjustbox}{max width=\linewidth}
    $
    \displaystyle
    P(Y_{\text{apple}}=1, C_{\text{red}}=1, C_{\text{edible}}=1 \mid \inlineimg{appler}) = P(Y_{\text{apple}}=0, C_{\text{red}}=1, C_{\text{edible}}=1 \mid \inlineimg{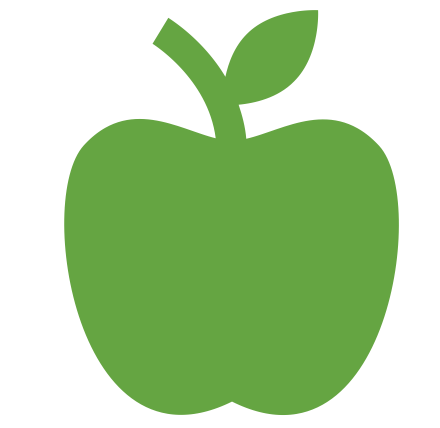})
    $
    \end{adjustbox}
    \end{equation*}
    This could be useful, if it is well-controlled, to achieve high task accuracy when concepts are insufficient (i.e., \textit{incomplete}).

    \item \textbf{Concept leakage}: This happens when a concept encodes information about other concepts~\citep{metric_paper}. A model $P(C_i, C_j \mid X)$ is subject to \emph{inter-concept leakage} with respect to the group actions $\mathfrak{b}$ on $X$ and $\rho_i(\mathfrak{b})$ on $C_i$ if:
    \begin{equation*}
    \begin{adjustbox}{max width=\linewidth}
    $
    \displaystyle
        \exists x \in X, \quad \exists \mathfrak{b} \in \mathfrak{G}, \quad \exists \rho_j(\mathfrak{b}): C \to C, \quad 
        P(C_i, C_j \mid \mathfrak{b} \cdot x) = P\big(\rho_i(\mathfrak{b}) \cdot C_i, \rho_j(\mathfrak{b}) \cdot C_j \mid x\big).
    $
    \end{adjustbox}
    \end{equation*}
    For instance, given image \inlineimg{appler}, changing pixel intensities does not change the concept ``red'', but changes the concept ``edible'':
    \begin{equation*}
    \begin{adjustbox}{max width=\linewidth}
    $
    \displaystyle
    P(C_{\text{red}}=1, C_{\text{edible}}=1 \mid \inlineimg{appler}) = P(C_{\text{red}}=1, C_{\text{edible}}=0 \mid \inlineimg{appleg})
    $
    \end{adjustbox}
    \end{equation*}
    In contrast to task leakage, concept leakage is always bad for alignment and, therefore, we argue, always undesirable.
\end{itemize}

\section{\texttt{PyC}: A Python Library for Interpretable Models} \label{app:library}
The proposed blueprint informs researchers about the key ingredients for building concept-based interpretable models. To support the implementation of existing models and the development of novel models, we designed a Python library with native support for concept-based data structures and processes. Our codebase is built on top of the popular PyTorch~\citep{pytorch} library to encourage the easy use and extensibility of our layers to arbitrary neural architectures. For more details on the codebase itself, please take a look at our code library 
(\url{https://github.com/pyc-team/pytorch_concepts}).

\section{Notes on inference equivariance} \label{app:inference-eq}

In addition to the consequences discussed in Section~\ref{sec:equiv}, inference equivariance enables us to highlight several further properties of the nature of interpretability:

\paragraph{Inference equivariance can be asymmetric:} Having a translation $\tau$ does not guarantee that an inverse translation $\tau^{-1}$ exists. However, the absence of a reverse transformation does not preclude our ability to verify inference equivariance.

\paragraph{Explanations are a form of selection:} An explanation of a system's behaviour can be seen as a process of selection, where conditioning on observed evidence picks out a specific subset from the system's complete conditional probability table. In our example in Table~1, when we observe a particular variable, say $X_1$, we effectively select a corresponding segment of the table that relates $X_1$ to $Y$. This selection -- formally represented with the distribution $P(Y^{(m)} \mid X^{(m)})$ -- encapsulates the explanation by narrowing down the myriad potential outcomes to the ones relevant to this observation.



\paragraph{Local vs. global equivariance:} Equivariance may hold over the entire state space of the system (global) or only in certain regions (local). Local equivariance indicates that while the system may be interpretable under specific conditions, \textit{its interpretability might not generalise across all possible configurations}. Recognising the distinction between local and global equivariance is crucial for assessing the robustness of a system's interpretability.

\paragraph{Post-hoc methods complicate rather than simplify interpretability:} When applying post-hoc interpretability techniques, such as using surrogate models to explain the original system~\cite{hinton2015distilling, zilke2016deepred} or so-called feature importance methods~\cite{lime, shap, og_saliency, integrated_gradients}, an additional layer of equivariance is required. Suppose we use a surrogate function $m'$ to better understand the function $m$. In that case, there must be a consistent mapping between the machine variables of the original system $(X^m, Y^m))$ and those of the surrogate model $(X^{m'}, Y^{m'})$ and another mapping from the surrogate model to our model $(X^{(h)}, Y^{(h)})$. Formally, 
both the original and surrogate systems must satisfy the inference equivariance conditions:
\[
    \begin{adjustbox}{max width=\linewidth}
    \begin{tikzcd}[column sep=normal, row sep=normal]
    X^{(m)} \arrow[r, "m"] \arrow[d, "\tau"'] & Y^{(m)} \arrow[d, "\tau"] \\
    X^{(m')} \arrow[r, "m'"] \arrow[d, "\tau'"'] & Y^{(m')} \arrow[d, "\tau'"] \\
    X^{(h)} \arrow[r, "h"']  & Y^{(h)}
    \end{tikzcd}
    \end{adjustbox}
\]
This requirement ensures that the surrogate model $m'$ faithfully reflects the behaviour of the original model $m$, thus preserving interpretability even when using post-hoc methods. Ultimately, the need to establish these additional mappings significantly complicates the interpretability process as we now need to verify two equivariance conditions instead of one.

\section{Notes on Semantic and Functional Transparency} \label{app:transparencies}

Previous works~\cite{geiger2024causal,marconato2023interpretability} focused primarily on semantic inference equivariance, emphasising that equivariance should hold on generative factors/concepts. However, less attention has been paid to the functions that describe the mappings between concepts to tasks; for a user to truly understand the underlying mechanisms, the structure of the function and its parameters must also satisfy inference equivariance, as illustrated in the following example.
\begin{example}
Consider the conditional model $P(Y \mid C; \Theta)$ where $Y$ follows a Gaussian distribution:
\[
P(Y=y \mid C=c; \Theta=\theta) = \mathcal{N}(y \mid \theta^\top c, \sigma^2) \coloneqq \frac{1}{\sqrt{2\pi\sigma^2}} \exp\!\left(-\frac{(y-\theta^\top c)^2}{2\sigma^2}\right).
\]
For this model to be fully interpretable, it is not enough for a human user to simply understand the data representation encoded in $Y$ and $C$. Instead, inference equivariance must extend to the functional structure and its parameters. In other words, users should be able to modify or update the parameters -- such as $\theta$ or $\sigma^2$, or even alter constants like replacing $2\pi$ with $3\pi$ -- and still verify that the same equivariant relations hold. This ensures that the model's underlying functional form remains transparent.\\
\end{example}

The intuition behind this is that functional structure and parameters are key components of interpretability, not just the data representations. To capture this formally, we can distinguish between variables representing data, $C$, and those describing the model's functional structure, $\Theta$. The complete model can then be expressed as $P(Y \mid C; \Theta)$. Inference equivariance should hold for both $C$, ensuring \textit{semantic transparency}, and for $\Theta$, ensuring \textit{functional transparency}.

\end{document}